\documentclass{article}

\usepackage[preprint]{neurips_2026}

\usepackage[utf8]{inputenc} 
\usepackage[T1]{fontenc}    
\usepackage{hyperref}       
\usepackage{url}            
\usepackage{booktabs}       
\usepackage{amsfonts}       
\usepackage{nicefrac}       
\usepackage{microtype}      
\usepackage{xcolor}         
\usepackage{graphicx}
\usepackage{amsmath}
\usepackage{amsthm}
\usepackage{algorithm}
\usepackage{algorithmic}
\newtheorem{theorem}{Theorem}
\newtheorem{lemma}{Lemma}
\newtheorem{proposition}{Proposition}
\usepackage{float}
\usepackage{wrapfig}
\usepackage{stfloats}
\usepackage{caption}
\usepackage{needspace}

\title{LPDP: Inference-Time Reward Control for Variable-Length DNA Generation with Edit Flows}
\author{
  Jeongchan Kim \quad Yunkyung Ko \quad Jong Chul Ye \\
  KAIST AI \\
  \texttt{\{jchan.kim, jong.ye\}@kaist.ac.kr} \\
}

\begin{document}

\maketitle

\begin{abstract}
{We study the application of recent Edit Flows for  inference-time reward control for DNA sequence generation.} 
Unlike most reward-guided DNA generation {frameworks}, 
which operate on fixed-length sequence spaces,  {Edit Flows have a potential to} generate variable-length DNA through biologically plausible insertion, deletion, and substitution operations.
{In particular,} we propose \emph{Local Perturbation {Discrete} Programming (LPDP)}, a training-free, intermediate-state and action-aware local re-solving operator for variable-length DNA edit-action generators at inference time.
{More specifically,} at each guided rollout step, LPDP scores one-step root edits, retains a near-best root band, and re-ranks each retained root by solving a bounded local discrete program around its child sequence.
This local program uses the typed geometry of edit actions to focus on coherent substitution, insertion, or deletion subgraphs, and aggregates local continuations with either a hard Max backup or a soft log-sum-exponential (LSE) backup.
{We instantiate LPDP in two regimes: front-loaded reward tilting for enhancer optimization, where early edits are critical for establishing global regulatory sequence structure, and back-loaded reward tilting for exon–intron–exon inpainting, where late edits fine-tune splice-boundary contexts. }
%
%
Across enhancer and splice benchmarks, LPDP improves frozen-oracle reward under comparable search budgets while preserving base-flow and sequence-distribution diagnostics.
To our knowledge, LPDP is the first inference-time reward-tilting framework built specifically for variable-length DNA edit flows with explicit insertion, deletion, and substitution actions.
\end{abstract}

\begin{figure*}[t]
    \centering
    \includegraphics[width=\textwidth]{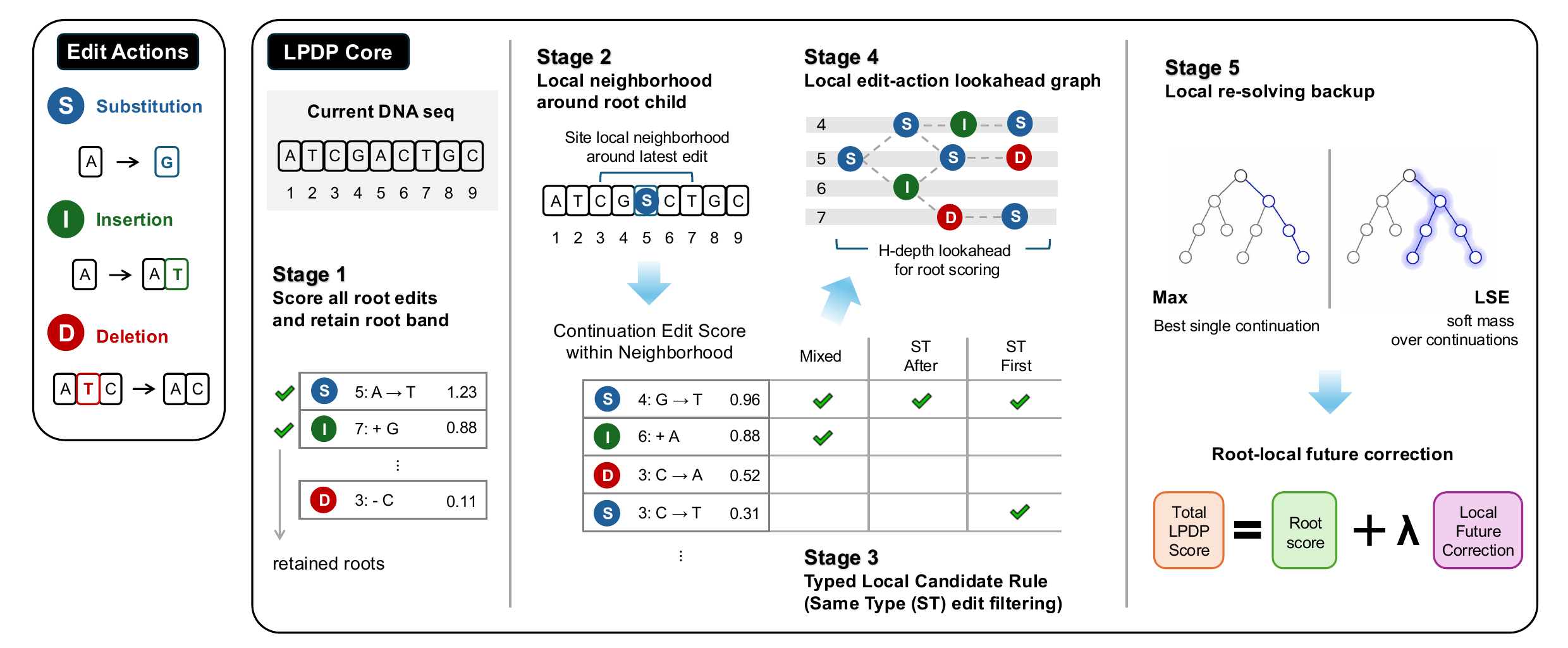}
    \caption{
    \textbf{Overview of LPDP.}
    LPDP re-solves each reward-guided rollout step at the level of edit actions.
    From the current sequence \(x_t\), it scores all valid substitution, insertion, and deletion root edits using the base-flow score and immediate oracle improvement, and retains a near-best root band for lookahead \textbf{(Stage 1)}.
    For each retained root, LPDP applies the root edit to obtain a child sequence and defines a site-local neighborhood around the latest edit site, within which continuation edits are scored \textbf{(Stage 2)}.
    It then applies a typed local candidate rule \textbf{(Stage 3)}:
    Mixed keeps the local top-\(K\) edits across all edit types, ST-after first ranks the mixed shortlist and then filters it by the same edit type as the current local parent edit, and ST-first first restricts to that edit type and then ranks within the typed subset.
    These candidates define a bounded \(H\)-depth local edit-action lookahead graph used only for root scoring \textbf{(Stage 4)}; the graph shown is a Mixed example.
    Finally, LPDP backs up this graph with Max, which selects the best single continuation, or LSE, which aggregates soft mass over plausible continuations \textbf{(Stage 5)}.
    It applies only the highest-scoring root edit and re-solves from the new sequence at the next rollout step.
    }
    \label{fig:lpdp_overview}
\end{figure*}

\section{Introduction}
\label{sec:introduction}

Reward-guided DNA design asks a generative model to propose DNA sequences with high predicted function while avoiding obvious departures from realistic sequence statistics.
{Specifically,} a sequence is functional if a downstream biological oracle predicts a desired effect, such as strong regulatory activity or correct splice-site formation; it is plausible if it still resembles realistic genomic DNA rather than an adversarial string optimized only for the oracle. 
We study this setting with frozen oracles, which are queried only at inference time and may be expensive or unavailable during generator training.

We consider two representative tasks. 
In enhancer optimization, the model seeks regulatory DNA sequences predicted to activate gene expression. 
In exon--intron--exon inpainting, the model fills an intronic region between two exon contexts and must create or preserve splice donor and acceptor signals needed for RNA splicing. 
These tasks stress complementary regimes: enhancer design depends on global regulatory sequence patterns, while splice inpainting depends on local boundary motifs and context.

Most reward-guided DNA generators are formulated over fixed-length sequence spaces, often using token replacement, masking, or denoising on \(X=\Sigma^N\)~\citep{dasilva2024dnadiffusion,wang2025drakes,dreamdna2026}. 
Fixed-length modeling is convenient, but many genomic design problems involve sequences that naturally occur at different lengths, such as regulatory elements, introns, and edit-induced variants. 
Insertions and deletions are therefore not merely nuisance operations; they are biologically meaningful changes that alter local sequence context. 
This motivates reward-guided generation methods that can operate directly over variable-length DNA.

Among recent generative models, Edit Flows~\citep{havasi2025editflows} provide such a backbone by generating finite sequences through explicit substitution, insertion, and deletion transitions. 
This action-level structure exposes more than a token sequence: at an intermediate rollout state, a candidate edit changes both the current sequence and the short local continuations that become available afterward. 
{This setting makes standard diffusion guidance difficult to apply directly: continuous guidance acts through gradients on fixed-dimensional noised states, while many discrete diffusion methods rely on fixed-position or masked-token updates.
In contrast, edit flows use valid edit actions that can change sequence length and coordinates.}
The challenge is to exploit this action geometry at inference time {to maximize the reward} without expanding the full edit tree.

{To address this challenge}, we propose \emph{Local Perturbation Discrete Programming (LPDP)}, a training-free, intermediate-state and action-aware reward-tilting method for pretrained DNA edit flows. 
LPDP treats each rollout decision as a root re-solving problem. 
At a current sequence \(x_t\), it scores all valid one-step root edits, retains a score band of competitive roots, and re-ranks each retained root by solving a bounded local discrete program around the child sequence induced by that root. 
Rather than searching a global frontier of partial sequences, LPDP asks a root-local question: if this edit is chosen now, what short local continuation makes it valuable?

{In particular, we consider two design axes for the local discrete program.} 
First, LPDP uses the typed geometry of edit actions to construct local candidate graphs. 
We compare \textbf{Mixed}, which keeps mixed-type local candidates, \textbf{ST-after}, which first forms a mixed local top-\(K\) shortlist and then filters by the previous edit type, and \textbf{ST-first}, which filters by edit type before ranking. 
These same-type (ST) rules define typed local subgraphs for bounded edit-flow re-solving. 
Second, LPDP aggregates local futures with either a hard Max backup, which follows the best local continuation, or a soft log-sum-exponential (LSE) backup, which aggregates multiple plausible continuations through a soft log-partition.

We instantiate LPDP in two regimes. 
For enhancer optimization, we apply front-loaded reward tilting, where early edits can shape the global regulatory sequence trajectory. 
For exon--intron--exon inpainting, we apply back-loaded reward tilting, where late edits act after a coarse intron context and splice-boundary candidates have formed. 
Across these tasks, LPDP improves frozen-oracle reward under comparable search budgets while preserving base-flow and sequence-distribution diagnostics. 

\section{Background and Related Work}
\label{sec:background_related}

\noindent\textbf{Fixed-length DNA generators and reward guidance.}
Modern DNA generators include masked language models~\citep{ji2021dnabert,dallatorre2025nucleotide}, discrete diffusion models~\citep{avdeyev2023ddsm,dasilva2024dnadiffusion,sarkar2024d3,yang2026d3lm}, and flow-matching variants~\citep{gat2024dfm,dreamdna2026}.
Existing reward-guided methods, including DRAKES~\citep{wang2025drakes}, SVDD-style guidance~\citep{li2024svdd}, and test-time diffusion refinement~\citep{uehara2025rewardguided}, mostly act on fixed-length token, mask, or denoising states.
LPDP targets a complementary setting: variable-length edit-flow trajectories with explicit insertion, deletion, and substitution actions.
Because these fixed-length guidance methods do not directly instantiate on the edit-action state space, we compare against search and particle baselines on the same frozen edit-flow generator.

\noindent\textbf{Edit flows and typed action geometry.}
Discrete flow matching defines continuous-time generative processes over discrete objects~\citep{gat2024dfm}. 
Edit Flows extend this view to sequence spaces with edit operations, using insertions, deletions, and substitutions as transitions over finite sequences~\citep{havasi2025editflows}. 
This transition structure matters for inference-time control. 
A substitution changes a base while preserving sequence length; an insertion or deletion changes length and shifts downstream coordinates. 
Thus edit actions induce a typed local geometry that is absent from fixed-position token replacement models. 
LPDP uses this geometry directly: it scores root edits, constructs local continuation graphs around retained roots, and applies a short discrete program on those graphs.

\noindent\textbf{Inference-time search and trajectory guidance.}
Our baselines represent standard ways to allocate test-time reward computation in discrete generation. 
CEM~\citep{rubinstein1999cem} refits a sampling distribution toward high-scoring sampled elites; beam search~\citep{freitag2017beam} expands a fixed-width frontier of partial trajectories; SMC~\citep{doucet2001smc} maintains and resamples a reward-weighted particle population; and TDS-style twisted samplers bias trajectories using twisting potentials~\citep{wu2023tds,whiteley2014twisted,heng2020controlled}. 
These methods spend computation across sampled trajectories, particles, or sequence-level frontiers. 
LPDP uses a different allocation rule: it first scores all one-step root edits exactly, retains competitive roots, and then re-solves a small local action graph around each retained root. 
Thus the contrast is not simply search versus no search, but \emph{sequence-level frontier breadth} versus \emph{root-conditioned local re-solving}.

\noindent\textbf{Soft discrete programming.}
The log-sum-exponential Bellman operator is standard in entropy-regularized control, path-integral control, and maximum causal entropy models~\citep{kappen2005pathintegral,todorov2009linearly,ziebart2010maximum}. 
In those settings, a soft value aggregates multiple high-return futures rather than selecting only one best path. 
LPDP-LSE uses the same operator on the restricted local graph induced by a retained root, while LPDP-Max corresponds to the low-temperature hard-selection limit. 
This lets LPDP compare two local backups: a hard exploitation backup over the single best local continuation and a soft backup over the tilted mass of multiple plausible continuations.

\section{Methods} 
\label{sec:method}

LPDP is a root-conditioned local re-solving operator for pretrained edit flows with frozen reward oracles.
Rather than allocating inference-time computation to a sequence-level frontier, LPDP spends its oracle budget locally around root edits that are already competitive under exact one-step scoring.
Figure~\ref{fig:lpdp_overview} summarizes the five stages used at each guided rollout step.
Importantly, the local lookahead graph is used only to re-rank the current root edit: after the backup, LPDP applies a single root edit and then repeats the same procedure from the new sequence.

\subsection{Stage 1: Score all root edits and retain root band}
\label{sec:edit_actions}
\label{sec:root_band}

We consider reward-guided generation with an Edit-Flow model~\citep{havasi2025editflows} pretrained on DNA sequences.
Let \(\mathcal X\) be the set of finite DNA sequences over alphabet \(\mathcal V=\{A,C,G,T\}\).
For any sequence \(x\in\mathcal X\), let \(\mathcal A(x)\) denote the set of edits that are valid at \(x\).
An edit action is written as \(a=(s,e,v)\in\mathcal A(x)\), where \(s\) is a valid site, \(e\in\{\mathrm{ins},\mathrm{del},\mathrm{sub}\}\) is the edit type, and \(v\in\mathcal V\) is the nucleotide used for insertion or substitution.
For deletion, the token field is ignored.
Applying an edit is deterministic; we write \(f(x,a)\) for the sequence obtained by applying \(a\) to \(x\).

At inference step \(t\), the pretrained edit-flow model defines a Continuous-Time Markov Chain (CTMC) rate over variable-length sequences, supported only on one-edit successors \(x'=f(x,a)\).
The rate \(u_t^\theta(x'\mid x)\) is the learned edit-flow transition rate from \(x\) to \(x'\) at time \(t\).
We convert these rates into a normalized base edit proposal,
\begin{equation}
\label{eq:base_edit_proposal}
    p_0(a\mid x,t)
    =
    \frac{u_t^\theta(f(x,a)\mid x)}
    {\sum_{a'\in\mathcal A(x)}u_t^\theta(f(x,a')\mid x)}.
\end{equation}
Thus \(p_0(a\mid x,t)\) is the base-model probability of choosing edit \(a\) among all valid one-edit transitions from \(x\) at time \(t\).
This exposes the edit type, site, and token of each transition as local structure for LPDP, without assuming that any edit type is globally dominant.

The reward oracle is denoted by \(R:\mathcal X\to\mathbb R\).
{Unlike noised or masked intermediate states in many diffusion samplers, an edit-flow intermediate state is itself a complete DNA sequence. Thus the frozen DNA oracle can be queried on both the current sequence and its one-edit successor.}
The immediate reward change caused by edit \(a\) is
\begin{equation}
    \Delta R(x,a)=R(f(x,a))-R(x).
\end{equation}

Given a reward scale \(\beta>0\), we define the one-step tilted score
\begin{equation}
\label{eq:root_score}
    q_t(x,a)
    =
    \log p_0(a\mid x,t)
    +
    \beta\,\Delta R(x,a).
\end{equation}
The first term keeps the edit close to the pretrained edit-flow model, while the second term favors edits that increase the reward.
LPDP evaluates every valid root edit using Eq.~\eqref{eq:root_score}.
Let
\begin{equation}
    q_t^\star(x_t)=\max_{a\in\mathcal A(x_t)} q_t(x_t,a)
\end{equation}
be the best one-step root score.
For a tolerance \(\delta\ge0\), define the uncapped score band
\begin{equation}
\label{eq:uncapped_root_band}
    \mathcal B_\delta(x_t)
    =
    \left\{
    a\in\mathcal A(x_t):
    q_t(x_t,a)\ge q_t^\star(x_t)-\delta
    \right\}.
\end{equation}
With a root cap \(K_{\mathrm{root}}\), the root band used by LPDP is
\begin{equation}
\label{eq:root_band}
    \mathcal B_{\delta,K_{\mathrm{root}}}(x_t)
    =
    \operatorname{TopK}_{K_{\mathrm{root}}}
    \left(\mathcal B_\delta(x_t);\ q_t(x_t,\cdot)\right),
\end{equation}
where the top-\(K\) operation is taken by the one-step score \(q_t\).
The word ``band'' refers to a score band around the best one-step edit, not to a genomic interval: the root band contains edits that are competitive under exact one-step scoring, regardless of where they occur in the sequence.
The band is LPDP's first approximation: rather than expanding lookahead at every valid root edit, LPDP spends lookahead only on near-best roots under exact one-step scoring.

\begin{lemma}[Root-band truncation]
\label{lem:root_band}
Every root edit outside the uncapped band \(\mathcal B_\delta(x_t)\) has one-step tilted score strictly below \(q_t^\star(x_t)-\delta\); restricting root lookahead to \(\mathcal B_\delta(x_t)\) thus discards only edits whose exact one-step score is at least \(\delta\) below the optimum.
\end{lemma}

The lemma is deliberately a root-level statement: LPDP trusts exact one-step scoring to identify the root edits worth spending lookahead, while a low one-step edit may in principle have unusually good future value.
The proof is in Appendix~\ref{app:proof_root_band}.

\subsection{Stage 2: Local neighborhood around root child}
\label{sec:local_neighborhood}

For each retained root edit \(a\in\mathcal B_{\delta,K_{\mathrm{root}}}(x_t)\), let
\(y_a=f(x_t,a)\) be the child sequence obtained after applying \(a\).
LPDP evaluates this root by asking a local continuation question: if \(a\) were chosen now, what short sequence of nearby edits would make this child valuable?

A continuation state \(z\) is reached by a sequence of local edits, and \(a_{\mathrm{prev}}\) denotes the most recent edit used to reach \(z\).
Let \(\alpha(a_{\mathrm{prev}})\) denote the anchor site of the previous edit in the current sequence: the substituted site for substitutions, and the nearby resulting site for insertions or deletions.
For radius \(r\), LPDP defines the site-local neighborhood
\begin{equation}
\label{eq:local_neighborhood}
    \mathcal N_r(z,a_{\mathrm{prev}})
    =
    \left\{
    (s,e,v)\in\mathcal A(z):
    |s-\alpha(a_{\mathrm{prev}})|\le r
    \right\}.
\end{equation}
This neighborhood specifies where LPDP is allowed to look after a retained root; the next stage specifies which candidates inside the neighborhood are kept.
For an internal lookahead depth \(i\), each continuation edit \(b\in\mathcal N_r(z,a_{\mathrm{prev}})\) is scored by the same tilted form,
\begin{equation}
\label{eq:local_q}
    q_{t+i}(z,b)
    =
    \log p_0(b\mid z,t+i)
    +
    \beta\left[R(f(z,b))-R(z)\right].
\end{equation}

\subsection{Stage 3: Typed local candidate rules}
\label{sec:same_type}

\begin{figure}[b]
    \centering
    \includegraphics[width=\linewidth]{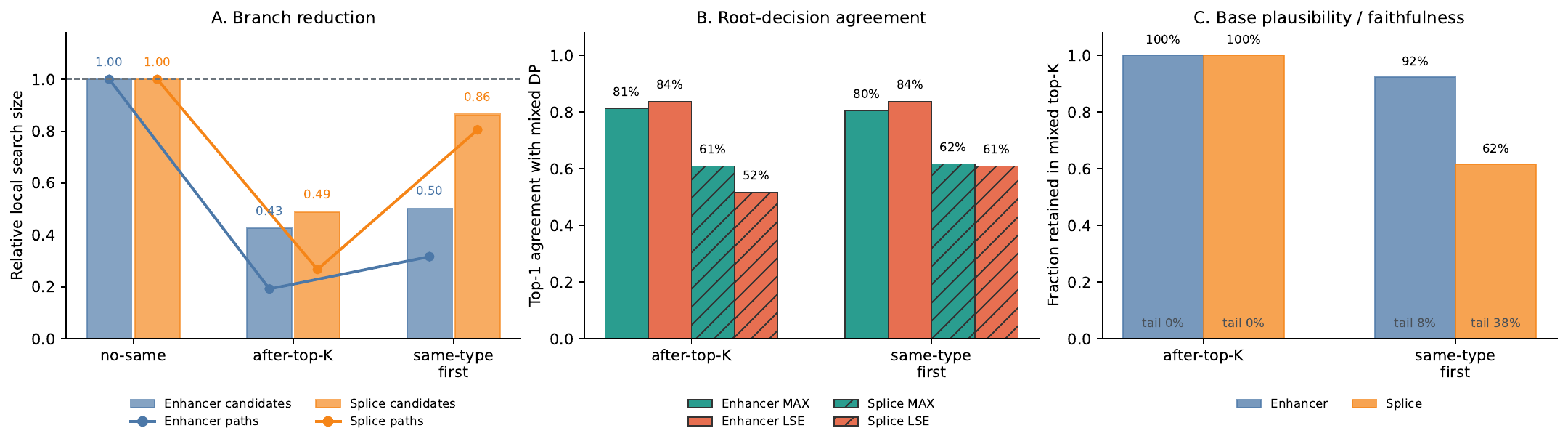}
    \caption{
    \textbf{Typed local candidate rules as structured local approximations.}
    We compare the Mixed candidate rule with two same-type (ST) pruning rules under the main 256-step schedule and 16-step guided window.
    (A) ST rules reduce the retained local search space relative to Mixed.
    (B) Despite this reduction, they often select the same top root edit as Mixed local DP.
    (C) ST-after is the conservative typed rule: all retained candidates remain inside the mixed local top-\(K\) shortlist, whereas ST-first can introduce candidates outside that shortlist.
    Full diagnostics, including path-mass efficiency and root-type breakdowns, are in Appendix~\ref{app:sametype_diagnostics}.
    }
    \label{fig:sametype_main}
\end{figure}

Let \(e(a)\) denote the edit type of action \(a\), and let \(K_{\mathrm{loc}}\) be the local cap.
A site-local neighborhood \(\mathcal N_r(z,a_{\mathrm{prev}})\) can still contain substitutions, insertions, and deletions after locality truncation.
LPDP therefore constructs the final local candidate set \(\mathcal C(z,a_{\mathrm{prev}})\subseteq\mathcal N_r(z,a_{\mathrm{prev}})\) through a small family of typed candidate rules.
We do not assume that same-type continuations are globally optimal.
Instead, we compare three candidate rules, abbreviated as \textbf{Mixed}, \textbf{ST-after}, and \textbf{ST-first}. Mixed is the reference local rule, while ST-after and ST-first are typed local subgraph approximations that trade off mixed-type flexibility, base-prior faithfulness, and local search size.

\noindent\textbf{Mixed.}
The mixed-type local set is
\begin{equation}
\label{eq:c_none}
    \mathcal C_{\mathrm{none}}(z,a_{\mathrm{prev}})
    =
    \operatorname{TopK}_{K_{\mathrm{loc}}}
    \left(\mathcal N_r(z,a_{\mathrm{prev}});\ p_0(\cdot\mid z,t)\right).
\end{equation}
This keeps the most likely local edits under the base flow, allowing substitutions, insertions, and deletions to compete together.

\noindent\textbf{Same-type after top-\(K\) (ST-after).}
First form the mixed local shortlist
\begin{equation}
    \widetilde{\mathcal C}(z,a_{\mathrm{prev}})
    =
    \mathcal C_{\mathrm{none}}(z,a_{\mathrm{prev}}).
\end{equation}
Then keep same-type edits inside that shortlist:
\begin{equation}
\label{eq:c_after}
    \mathcal C_{\mathrm{after}}(z,a_{\mathrm{prev}})
    =
    \left\{b\in\widetilde{\mathcal C}(z,a_{\mathrm{prev}}): e(b)=e(a_{\mathrm{prev}})\right\},
\end{equation}
with fallback to \(\widetilde{\mathcal C}\) if the filtered set is empty.
ST-after first asks which local edits are base-plausible, and only then restricts to the typed subgraph.

\noindent\textbf{Same-type first (ST-first).}
The ST-first set is
\begin{equation}
\label{eq:c_first}
    \mathcal C_{\mathrm{first}}(z,a_{\mathrm{prev}})
    =
    \operatorname{TopK}_{K_{\mathrm{loc}}}
    \left(
    \left\{b\in\mathcal N_r(z,a_{\mathrm{prev}}):e(b)=e(a_{\mathrm{prev}})\right\};\ p_0(\cdot\mid z,t)
    \right),
\end{equation}
again with fallback to Eq.~\eqref{eq:c_none} if the same-type set is empty.
ST-first commits to the typed neighborhood before ranking, and is therefore the stronger typed subproblem.

\begin{proposition}[Conservative ST-after pruning]
\label{prop:after_topk}
Every candidate \(b\in\mathcal C_{\mathrm{after}}(z,a_{\mathrm{prev}})\) has mixed-type local base-prior rank at most \(K_{\mathrm{loc}}\) within \(\mathcal N_r(z,a_{\mathrm{prev}})\).
In contrast, a candidate selected by ST-first can have mixed-type local rank greater than \(K_{\mathrm{loc}}\).
\end{proposition}

Proposition~\ref{prop:after_topk} identifies ST-after as the conservative member of the typed candidate-rule family: it imposes type coherence without introducing edits outside the base mixed top-\(K\) shortlist.
ST-first instead commits to the typed neighborhood before ranking, which can be useful when the reward favors a more focused edit mode.
We therefore report Mixed, ST-after, and ST-first as complementary candidate-rule choices throughout the experiments.
Figure~\ref{fig:sametype_main} summarizes whether these rules reduce local search, preserve the Mixed local-DP root decision, and remain faithful to the base mixed top-\(K\) shortlist.

\subsection{Stage 4: Local edit-action lookahead graph}
\label{sec:local_graph}

Using the candidate rule from Stage~3, LPDP builds a bounded root-local lookahead graph for each retained root.
Mathematically, this graph is rooted at \(y_a=f(x_t,a)\).
A node is an intermediate sequence \(z\), and an edge \(z\to f(z,b)\) is a valid continuation edit \(b\in\mathcal C(z,a_{\mathrm{prev}})\).
Conditioning the candidate set on \(a_{\mathrm{prev}}\) keeps the graph anchored to the most recent local edit.

Let \(H\) be the total LPDP lookahead horizon.
After a root edit is fixed for evaluation, the graph is expanded for at most \(H-1\) continuation edits, so \(H=1\) corresponds to pure one-step root scoring and \(H>1\) adds hypothetical local continuations below the root.
This graph is not a multi-edit commitment: it is used only to compute a root-local future correction, after which LPDP applies a single root edit and re-solves at the next rollout step.
The graph remains tractable because the root band, site-local radius, local cap, and horizon are all bounded.

\subsection{Stage 5: Local re-solving backup}
\label{sec:max_lse_backup}

Given a bounded local graph, LPDP computes a root-local continuation value with either a soft LSE backup or a hard Max backup.
The terminal value is \(V_0(z,a_{\mathrm{prev}})=0\).

\noindent\textbf{Soft LSE backup.}
Following entropy-regularized control~\citep{kappen2005pathintegral,todorov2009linearly,ziebart2010maximum}, we replace the hard maximum over local continuations by a soft log-partition.
The LSE backup takes the soft Bellman form:
\begin{equation}
\label{eq:lse_backup}
V_h^{\mathrm{lse}}(z,a_{\mathrm{prev}})
=
\tau
\log
\sum_{b\in\mathcal C(z,a_{\mathrm{prev}})}
\exp\left(
\frac{q(z,b)+\gamma V_{h-1}^{\mathrm{lse}}(f(z,b),b)}{\tau}
\right),
\end{equation}
where \(\tau>0\) is the soft backup temperature and \(\gamma\in[0,1]\) is a local lookahead discount; \(\mathcal C\), \(q\), and \(f\) are as defined in Stages~1--4.
This backup assigns high value to a root when its restricted local graph contains substantial reward-tilted mass over short continuations, rather than only a single high-scoring continuation.
When \(\gamma=1\), the recursion is exactly the log-partition function over all length-\(h\) edit paths in the restricted local graph; Theorem~\ref{thm:path_partition} in Appendix~\ref{app:proof_partition} formalizes this path-space interpretation.

\noindent\textbf{Hard Max backup.}
The Max backup is
\begin{equation}
\label{eq:max_backup}
V_h^{\max}(z,a_{\mathrm{prev}})
=
\max_{b\in\mathcal C(z,a_{\mathrm{prev}})}
\left[
q(z,b)+\gamma V_{h-1}^{\max}(f(z,b),b)
\right],
\end{equation}
which follows the single best local continuation inside the same restricted graph.
By the standard \(\tau\to0^+\) limit of the log-sum-exponential operator, Max is the low-temperature limit of LSE.
When \(\gamma=1\), the difference between the LSE and Max values is bounded by the soft path entropy term \(\tau\log|\mathcal T_h(z,a_{\mathrm{prev}})|\), where \(\mathcal T_h(z,a_{\mathrm{prev}})\) denotes the finite set of length-\(h\) local paths in the restricted graph (Theorem~\ref{thm:max_limit}, Appendix~\ref{app:proof_max_limit}).
Max and LSE are therefore two backup regimes for the same restricted local discrete program: Max exploits the single best local future, while LSE preserves soft path mass over multiple plausible futures.

\noindent\textbf{Sensitivity to typed pruning.}
Typed candidate rules and backup operators affect different parts of LPDP.
The candidate rule determines which local edit-action graph is solved, whereas the backup determines how paths on that graph are aggregated.
Under Max, pruning the mixed graph leaves the local value unchanged whenever the best mixed path is still retained (Proposition~\ref{prop:sametype_recovery}, Appendix~\ref{app:proof_pruning}).
Under LSE, even if the best path remains, pruning can change the value by removing reward-tilted mass from other plausible paths, creating a log-partition gap (Proposition~\ref{prop:soft_pruning_gap}, Appendix~\ref{app:proof_pruning}).
Thus, typed candidate rules define the local graph approximation, while Max and LSE define the aggregation regime on that graph.

\noindent\textbf{Root-local future correction and final score.}
Given either backup, LPDP scores a root edit \(a\) by
\begin{equation}
\label{eq:lpdp_root_score}
    S_{\mathrm{LPDP}}(x_t,a)
    =
    q_t(x_t,a)
    +
    \lambda V^{\mathrm{loc}}_{H-1}(f(x_t,a),a),
\end{equation}
where \(V^{\mathrm{loc}}\) is either \(V^{\max}\) or \(V^{\mathrm{lse}}\).
Here \(\beta\) controls the reward tilt in the root and continuation scores, while \(\lambda\ge0\) damps the approximate local future correction computed on the truncated graph.
LPDP selects the root edit with the largest \(S_{\mathrm{LPDP}}\) among edits in \(\mathcal B_{\delta,K_{\mathrm{root}}}(x_t)\).

This final score also gives LPDP-LSE a restricted soft-control interpretation.
Entropy-regularized control favors high-reward trajectories while penalizing deviations from a reference process, yielding soft values that are log-partitions over reward-tilted reference paths.
LPDP-LSE has the same structure only on the restricted root-local graph: it aggregates base-flow path mass tilted by exponentiated reward increments to estimate a local future correction.

If \(H=1\) or \(\lambda=0\), the continuation term vanishes and LPDP reduces to exact one-step planning over the root band.
At the other extreme, with \(\lambda=1\), the full root action set, and 
\(\mathcal C(z,a_{\mathrm{prev}})=\mathcal A(z)\) at every intermediate state, 
LPDP recovers finite-horizon dynamic programming on the full edit graph.
Otherwise, LPDP is exact root scoring plus a damped local future correction.

\subsection{The algorithm}
Algorithm~\ref{alg:lpdp_compact} combines the five stages into the per-step LPDP procedure: at rollout step \(t\), evaluate all valid root edits exactly via Eq.~\eqref{eq:root_score}, retain the score band of near-best roots via Eq.~\eqref{eq:root_band}, build a bounded local graph around each retained child using Eqs.~\eqref{eq:local_neighborhood} and \eqref{eq:c_none}--\eqref{eq:c_first}, back it up with Eq.~\eqref{eq:max_backup} or Eq.~\eqref{eq:lse_backup}, and apply the maximizer of Eq.~\eqref{eq:lpdp_root_score}.
Rolling this one-edit re-solving step out generates the full sequence.

\begin{algorithm}[h]
\caption{Local Perturbation Discrete Programming (LPDP)}
\label{alg:lpdp_compact}
\begin{algorithmic}[1]
\REQUIRE Current sequence \(x_t\); root-band width \(\delta\); root cap \(K_{\mathrm{root}}\); local radius \(r\); local cap \(K_{\mathrm{loc}}\); horizon \(H\); correction damping \(\lambda\); candidate rule \(\rho\in\{\mathrm{mixed},\mathrm{after},\mathrm{first}\}\); backup \(\oplus\in\{\mathrm{lse},\max\}\).
\STATE Evaluate \(q_t(x_t,a)\) for all \(a\in\mathcal A(x_t)\) using Eq.~\eqref{eq:root_score}.
\STATE Construct the root band \(\mathcal B_{\delta,K_{\mathrm{root}}}(x_t)\) using Eq.~\eqref{eq:root_band}.
\FOR{each \(a\in\mathcal B_{\delta,K_{\mathrm{root}}}(x_t)\)}
    \STATE Set \(y_a=f(x_t,a)\).
    \STATE Construct the local candidate graph rooted at \(y_a\) using Eq.~\eqref{eq:local_neighborhood} and the candidate rule \(\rho\) from Eqs.~\eqref{eq:c_none}--\eqref{eq:c_first}.
    \STATE Compute \(V_{H-1}^{\oplus}(y_a,a)\) using Eq.~\eqref{eq:max_backup} if \(\oplus=\max\) and Eq.~\eqref{eq:lse_backup} if \(\oplus=\mathrm{lse}\).
    \STATE Compute \(S_{\mathrm{LPDP}}(x_t,a)\) using Eq.~\eqref{eq:lpdp_root_score}.
\ENDFOR
\STATE \textbf{return}
\[
\operatorname*{arg\,max}_{a \in \mathcal{B}_{\delta,K_{\mathrm{root}}}(x_t)} S_{\mathrm{LPDP}}(x_t, a).
\]
\end{algorithmic}
\end{algorithm}
\vspace{-0.2em}
Across tasks, LPDP uses the same root re-solving recursion, candidate rules, and Max/LSE backups; task-specific choices enter only through the conditioning context, the frozen reward oracle, and the rollout window in which guidance is applied.

\section{Experiments}
\label{sec:experiments}

\noindent\textbf{Reward models and baselines.}
We evaluate LPDP on two reward-guided DNA edit-flow tasks with frozen generators and reward models.
Enhancer optimization uses a pretrained gReLU oracle~\citep{lal2025grelu,wang2025drakes} to optimize predicted HepG2 activity, while exon--intron--exon splice inpainting uses a frozen SpliceAI reward model~\citep{jaganathan2019spliceai}.
These tasks test complementary reward regimes: global regulatory sequence patterns for enhancers and local donor--acceptor boundary formation for splicing.
All main benchmarks use 500 generated samples per method and a 256-step base edit-flow schedule.
Guidance is applied during the first 16 rollout steps for enhancer design and during the last 16 steps for splice inpainting, where coarse intron context and candidate splice-boundary neighborhoods have already formed.
Baseline parameters are chosen for broadly comparable effective search depth and uncached oracle-call scale; exact settings are given in Appendix~\ref{app:reproducibility}. Appendix~\ref{app:appendix_suites} reports sensitivity ablations over guidance windows, \(\lambda\), and \(H\), supporting the defaults as cost-controlled operating points.

\textbf{Metrics.}
For enhancer optimization, the primary metric is predicted HepG2 activity; 
we additionally report 3-mer JSD (k-mer drift), JASPAR motif-profile 
correlation, and base trajectory log-likelihood (base-flow compatibility). 
The UMAP panel in Figure~\ref{fig:enhancer_main} is qualitative.
For splice inpainting, we report Splice-Geomean (geometric mean of donor 
and acceptor scores), Splice-Min (weaker of the two), Donor GT Rate 
(preservation of the canonical GT donor motif), and base trajectory 
log-likelihood; full definitions are in Appendix~\ref{app:metric_definitions}.
Calls/sample counts uncached oracle evaluations per sample and serves 
as a cost descriptor; see Appendix~\ref{app:reproducibility} for details.

\noindent\textbf{Enhancer optimization.}
The enhancer dataset consists of GRCh38-derived variable-length sequences of 200--400 bp. 
Each sequence is evaluated by five 200-bp windows, and the task label is the mean of the top two HepG2 oracle predictions across windows. 
We compare raw base sampling, CEM, beam search, TDS, SMC, and LPDP variants under the same 256-step schedule with a 16-step guided window.

\vspace{-8pt}
\begin{table}[h]
\centering
\small
\caption{Enhancer main results. \textbf{Bold}: best, and \underline{underline}: second-best result.}
\label{tab:enhancer_main_all}
\resizebox{\linewidth}{!}{
\begin{tabular}{lccccc}
\toprule
Method & Pred-Activity $\uparrow$ & 3-mer JSD $\downarrow$ & JASPAR Corr $\uparrow$ & Base traj.-LL $\uparrow$ & Calls/sample \\
\midrule
Raw base & 1.280 & 0.02234 & 0.9886 & $-$7.831 & 0 \\
\midrule
CEM  & 6.301 & 0.02516 & 0.9867 & \underline{$-$7.452} & 1,141 \\
Beam & 7.116 & 0.02362 & 0.9863 & \textbf{$-$7.385} & 40,388 \\
TDS  & 7.082 & 0.02386 & 0.9860 & $-$7.463 & 39,822 \\
SMC  & 5.333 & 0.02486 & 0.9865 & $-$7.534 & 40,214 \\
\midrule
LPDP-Mixed-Max         & 7.803 & 0.02332 & 0.9846 & $-$7.798 & 39,761 \\
LPDP-Mixed-LSE         & 7.813 & 0.02366 & 0.9845 & $-$7.803 & 39,769 \\
LPDP-ST-after-Max      & \underline{7.814} & \underline{0.02278} & \underline{0.9872} & $-$7.853 & {39,652} \\
LPDP-ST-after-LSE      & \textbf{7.837} & \textbf{0.02253} & \textbf{0.9873} & $-$7.874 & {39,693} \\
LPDP-ST-first-Max      & 7.727 & 0.02387 & 0.9856 & $-$7.813 & 39,726 \\
LPDP-ST-first-LSE      & 7.723 & 0.02425 & 0.9864 & $-$7.823 & 39,801 \\
\bottomrule
\end{tabular}
}
\end{table}

\begin{figure*}[b]
    \centering
    \includegraphics[width=\textwidth]{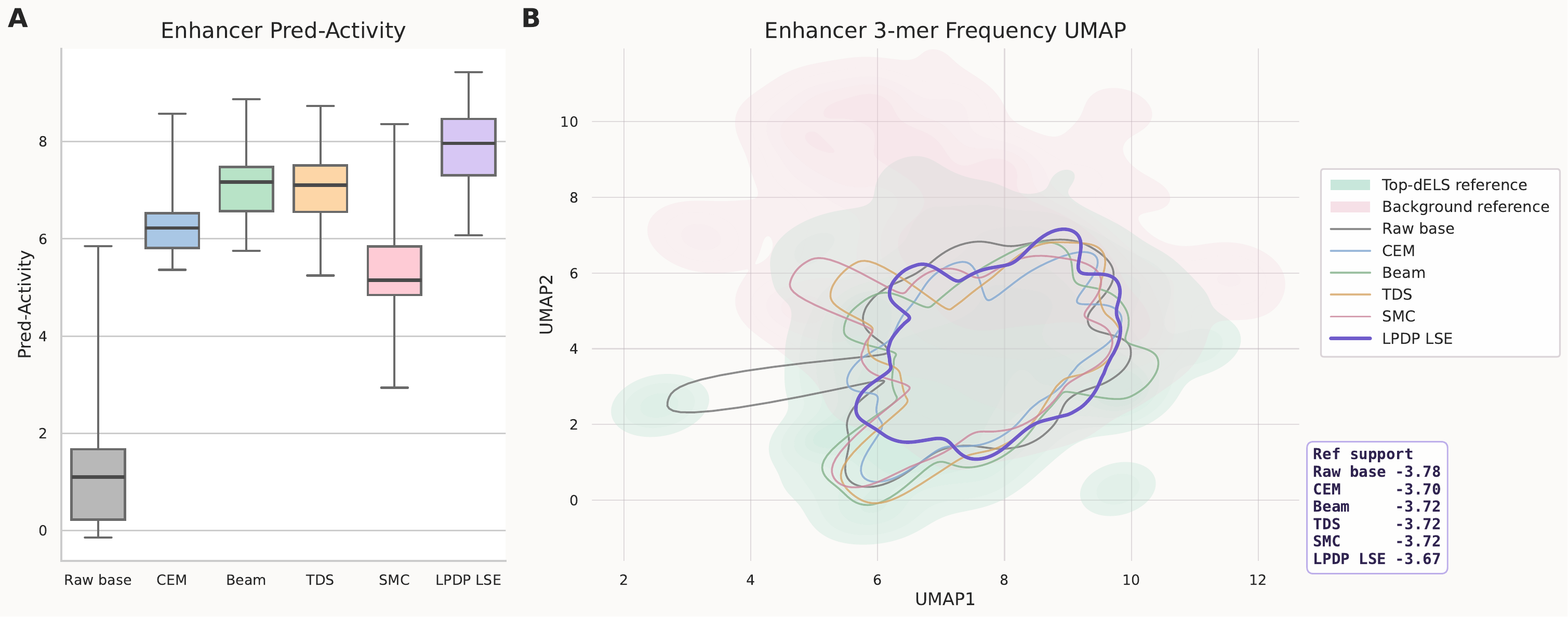}
    \caption{\textbf{Enhancer activity and 3-mer support.}
    (A) Predicted HepG2 activity distributions for generated enhancer sequences. LPDP-ST-after-LSE shifts the distribution toward higher activity. 
    (B) Qualitative 3-mer UMAP visualization. Reference points correspond to top-dELS and background sequences; method contours show generated samples.}
    \label{fig:enhancer_main}
\end{figure*}

LPDP-ST-after-LSE gives the highest predicted HepG2 activity and the lowest 3-mer JSD among the compared guided samplers, while using an oracle-call scale comparable to beam, TDS, and SMC. 
The distribution in Figure~\ref{fig:enhancer_main} shows that the activity gain is broad rather than driven by a few outliers. 
The quantitative 3-mer JSD indicates that the gain is not accompanied by large global k-mer drift; the UMAP panel is included only as a qualitative support visualization.
The comparison between Mixed and same-type variants shows that typed local candidate rules change the search bias without increasing the cost scale.

\begin{wrapfigure}{r}{0.35\columnwidth}
    \centering
    \vspace{-10pt}
    \includegraphics[width=\linewidth]{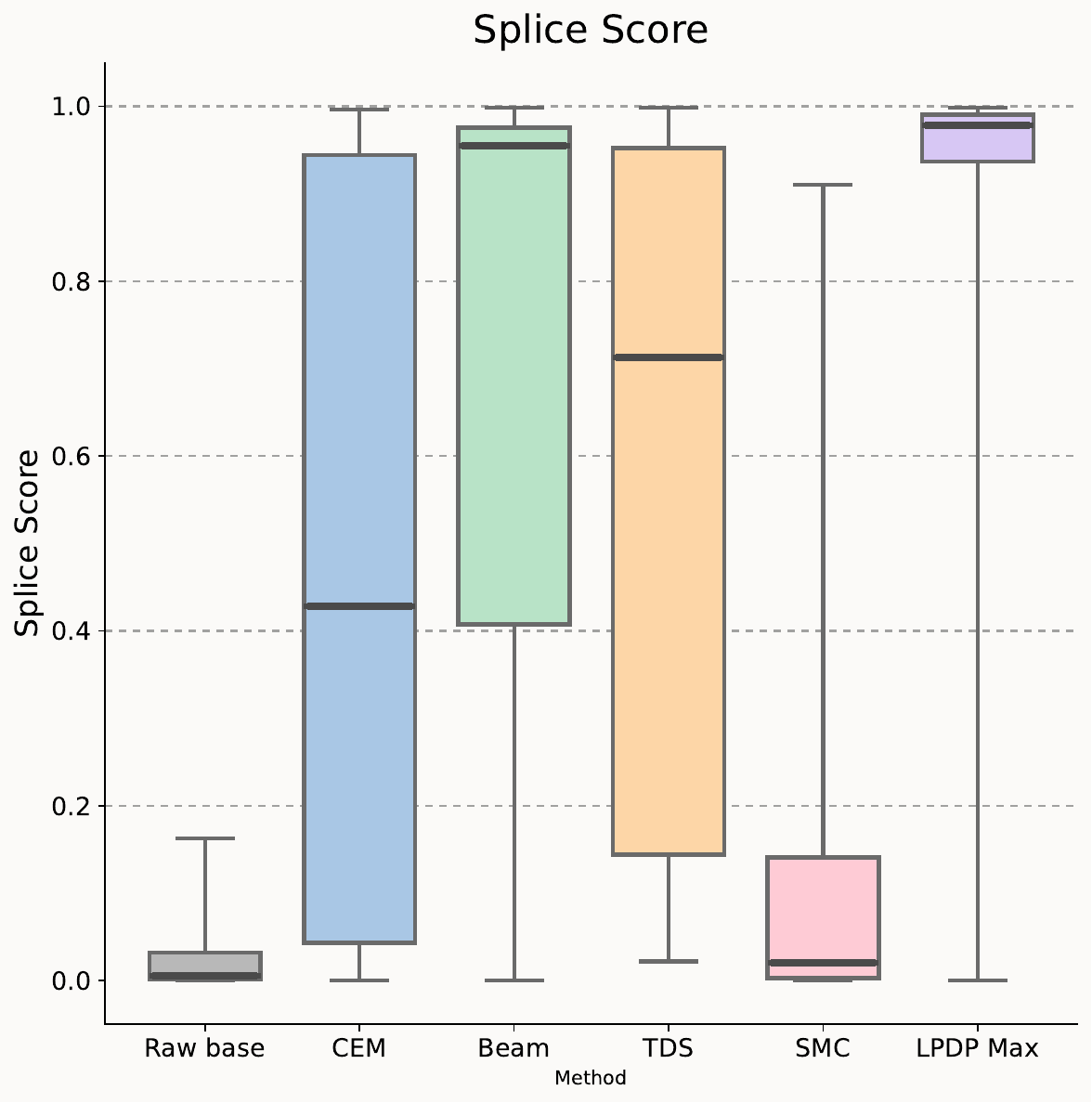}
    \vspace{-10pt}
    \caption{\textbf{Splice inpainting score distributions.} The box plot shows the distribution of intended-junction splice geomean scores.}
    \vspace{-10pt}
    \label{fig:splice_scores}
\end{wrapfigure}

\noindent\textbf{Splice inpainting.}
For splice inpainting, the model receives two exon contexts and generates the intervening intron. 
We apply LPDP over the late rollout window, where the intron context has partially formed and donor/acceptor boundary candidates are present, making the splice reward more actionable. 
LPDP uses the same root re-solving procedure as in enhancer design; only the conditioning context and frozen reward oracle change. 
Dataset construction and splice metrics are described in Appendix~\ref{app:splice_data} and Appendix~\ref{app:splice_metrics}.

The splice benchmark shows the same LPDP mechanism in a different reward geometry. 
ST-first-Max achieves the highest splice geomean, splice minimum, and donor GT rate, while ST-after-Max is the next strongest variant and has the best base trajectory likelihood. 
Together with the enhancer results, this pattern supports the interpretation of ST-after and ST-first as complementary typed candidate rules.
ST-after is a conservative base-faithful approximation that pairs naturally with LSE when multiple plausible regulatory continuations matter, whereas ST-first defines a stronger typed subproblem that can be advantageous for late-stage splice repair with a more committed local edit mode.

\begin{table}[h]
\centering
\small
\caption{Splice inpainting main results. \textbf{Bold}: best, and \underline{underline}: second-best result.}
\label{tab:splice_main}
\resizebox{\linewidth}{!}{
\begin{tabular}{lccccc}
\toprule
Method & Splice-Geomean $\uparrow$ & Splice-Min $\uparrow$ & Donor GT Rate $\uparrow$ & Base traj.-LL $\uparrow$ & Calls/sample \\
\midrule
Raw base & 0.0204 & 0.0060 & 0.344 & $-5.82$ & 0 \\
\midrule
CEM      & 0.4810 & 0.4395 & 0.719 & $-6.12$ & 21,805 \\
Beam     & 0.7100 & 0.6701 & 0.813 & $-6.11$ & 20,013 \\
TDS      & 0.5686 & 0.5058 & 0.719 & $-6.17$ & 21746 \\
SMC      & 0.1274 & 0.0728 & 0.500 & $-6.52$ & 24,211 \\
\midrule
LPDP-Mixed-Max     & 0.8227 & 0.7978 & 0.814 & $-6.01$ & 20,642 \\
LPDP-Mixed-LSE     & 0.8366 & 0.8017 & 0.813 & $-5.98$ & 22,088 \\
LPDP-ST-after-Max  & \underline{0.8664} & \underline{0.8348} & \underline{0.843} & \textbf{$-$5.73} & 20,997 \\
LPDP-ST-after-LSE  & 0.8385 & 0.8079 & 0.750 & $-5.93$ & 21,825 \\
LPDP-ST-first-Max  & \textbf{0.8713} & \textbf{0.8400} & \textbf{0.881} & \underline{$-5.76$} & 21,222 \\
LPDP-ST-first-LSE  & 0.8484 & 0.8216 & 0.823 & $-5.91$ & 22,246 \\
\bottomrule
\end{tabular}
}
\end{table}
\section{Conclusion}
\label{sec:conclusion}
LPDP provides inference-time reward tilting for variable-length DNA edit flows by re-ranking root edits with bounded local dynamic programs.
The method uses the structure exposed by edit flows---site, type, and token---to allocate uncached oracle computation to locally plausible futures around high-scoring roots.
In enhancer optimization, LPDP improves predicted activity under comparable search budgets while preserving base-flow and sequence-distribution diagnostics.
In splice inpainting, the same root-scoring and local re-solving core improves intended-junction quality under a different conditioning context and reward oracle.
Overall, the typed-rule results support viewing Mixed, ST-after, and ST-first as a candidate-rule family rather than task-specific heuristics.

LPDP remains limited by its root-band truncation, bounded site-local graph, and shallow lookahead, which keep inference tractable but may miss long-range or multi-site edit interactions.
Because all rewards are frozen predictive oracles, the reported biological improvements are model-based and require downstream experimental validation.

\medskip
{\small
\bibliographystyle{plainnat}
\bibliography{reference}
}

\newpage
\appendix
\textbf{\LARGE Appendices}

The appendix provides formal statements for LPDP, diagnostic definitions for same-type candidate rules, dataset and metric details, additional ablations, and reproducibility settings. We keep notation consistent with the main text.

\section{Proofs}
\label{app:proofs}

\subsection{Proof of Lemma~\ref{lem:root_band}}
\label{app:proof_root_band}

\begin{proof}
By definition,
\begin{equation}
    \mathcal B_\delta(x)
    =
    \{a:q_t(x,a)\ge q_t^\star(x)-\delta\}.
\end{equation}
Therefore, if $a\notin\mathcal B_\delta(x)$, then
\begin{equation}
    q_t(x,a)<q_t^\star(x)-\delta.
\end{equation}
Thus any action outside the uncapped root band is at least $\delta$ worse than the exact one-step optimum in immediate root score.  This proves the claim.  The top-$K_{\mathrm{root}}$ cap used in Eq.~\eqref{eq:root_band} is a separate computational cap applied after forming the band.
\end{proof}

\subsection{Theorem~\ref{thm:path_partition}: LPDP-LSE as a local path-space partition function}
\label{app:proof_partition}

\begin{theorem}[LPDP-LSE is a local path-space partition function]
\label{thm:path_partition}
Assume \(\gamma=1\). 
For any local state \(z\), previous edit \(a_{\mathrm{prev}}\), and horizon \(h\), let 
\(\mathcal T_h(z,a_{\mathrm{prev}})\) be the set of all length-\(h\) local edit paths obtained by repeatedly choosing continuation edits from the local candidate rule \(\mathcal C\). 
A path is written as
\[
\xi=(b_0,\ldots,b_{h-1}),
\]
with state sequence \(z_0=z\) and \(z_{i+1}=f(z_i,b_i)\). 
Define its cumulative local score by
\[
G_h(\xi;z,a_{\mathrm{prev}})
=
\sum_{i=0}^{h-1} q(z_i,b_i).
\]
Then the LSE value satisfies
\begin{equation}
\label{eq:partition_theorem}
V_h^{\mathrm{lse}}(z,a_{\mathrm{prev}})
=
\tau\log
\sum_{\xi\in\mathcal T_h(z,a_{\mathrm{prev}})}
\exp\!\left(
\frac{G_h(\xi;z,a_{\mathrm{prev}})}{\tau}
\right).
\end{equation}
\end{theorem}

\begin{proof}
We prove the claim by induction on the local planning horizon \(h\). 
For compactness, define the local path partition function
\begin{equation}
\label{eq:local_partition_appendix}
Z_h(z,a_{\mathrm{prev}})
:=
\sum_{\xi\in\mathcal T_h(z,a_{\mathrm{prev}})}
\exp\!\left(
\frac{G_h(\xi;z,a_{\mathrm{prev}})}{\tau}
\right).
\end{equation}
The theorem is equivalent to showing that
\[
V_h^{\mathrm{lse}}(z,a_{\mathrm{prev}})
=
\tau \log Z_h(z,a_{\mathrm{prev}})
\]
for every \(h\).

\paragraph{Base case.}
When \(h=0\), the path set \(\mathcal T_0(z,a_{\mathrm{prev}})\) contains only the empty path. 
This path applies no continuation edits, so its cumulative score is \(G_0=0\). 
Therefore
\[
Z_0(z,a_{\mathrm{prev}})
=
\exp(0/\tau)
=
1.
\]
Taking \(\tau\log\) gives
\[
\tau\log Z_0(z,a_{\mathrm{prev}})
=
\tau\log 1
=
0.
\]
This matches the terminal condition of the LSE recursion,
\[
V_0^{\mathrm{lse}}(z,a_{\mathrm{prev}})=0.
\]
Thus the claim holds for \(h=0\).

\paragraph{Inductive step.}
Assume the claim holds for horizon \(h-1\), for every possible local state \(z'\) and previous edit \(a'\):
\[
V_{h-1}^{\mathrm{lse}}(z',a')
=
\tau\log Z_{h-1}(z',a').
\]
We prove the statement for horizon \(h\).

Consider any path \(\xi\in\mathcal T_h(z,a_{\mathrm{prev}})\). 
By definition, its first edit must be some
\[
b\in\mathcal C(z,a_{\mathrm{prev}}).
\]
After applying \(b\), the next local state is \(f(z,b)\), and the previous edit for the suffix becomes \(b\). 
The remaining \(h-1\) edits therefore form a suffix path
\[
\xi'\in\mathcal T_{h-1}(f(z,b),b).
\]
Thus every length-\(h\) path can be uniquely decomposed as
\[
\xi=(b,\xi'),
\]
where \(b\in\mathcal C(z,a_{\mathrm{prev}})\) and 
\(\xi'\in\mathcal T_{h-1}(f(z,b),b)\). 
Conversely, every such pair \((b,\xi')\) defines a valid length-\(h\) path. 
This one-to-one decomposition lets us rewrite the sum over paths as a sum over first edits and suffix paths.

Since \(\gamma=1\), the cumulative score decomposes additively:
\begin{equation}
\label{eq:path_score_decomp_appendix}
G_h(\xi;z,a_{\mathrm{prev}})
=
q(z,b)
+
G_{h-1}(\xi';f(z,b),b).
\end{equation}
Substituting this decomposition into the partition function gives
\begin{align}
Z_h(z,a_{\mathrm{prev}})
&=
\sum_{\xi\in\mathcal T_h(z,a_{\mathrm{prev}})}
\exp\!\left(
\frac{G_h(\xi;z,a_{\mathrm{prev}})}{\tau}
\right)\\
&=
\sum_{b\in\mathcal C(z,a_{\mathrm{prev}})}
\sum_{\xi'\in\mathcal T_{h-1}(f(z,b),b)}
\exp\!\left(
\frac{
q(z,b)+G_{h-1}(\xi';f(z,b),b)
}{\tau}
\right)\\
&=
\sum_{b\in\mathcal C(z,a_{\mathrm{prev}})}
\exp\!\left(\frac{q(z,b)}{\tau}\right)
\sum_{\xi'\in\mathcal T_{h-1}(f(z,b),b)}
\exp\!\left(
\frac{G_{h-1}(\xi';f(z,b),b)}{\tau}
\right).
\end{align}
The inner sum is precisely the suffix partition function
\[
Z_{h-1}(f(z,b),b).
\]
By the induction hypothesis,
\[
Z_{h-1}(f(z,b),b)
=
\exp\!\left(
\frac{V_{h-1}^{\mathrm{lse}}(f(z,b),b)}{\tau}
\right).
\]
Therefore
\begin{align}
Z_h(z,a_{\mathrm{prev}})
&=
\sum_{b\in\mathcal C(z,a_{\mathrm{prev}})}
\exp\!\left(\frac{q(z,b)}{\tau}\right)
\exp\!\left(
\frac{V_{h-1}^{\mathrm{lse}}(f(z,b),b)}{\tau}
\right)\\
&=
\sum_{b\in\mathcal C(z,a_{\mathrm{prev}})}
\exp\!\left(
\frac{
q(z,b)+V_{h-1}^{\mathrm{lse}}(f(z,b),b)
}{\tau}
\right).
\end{align}
Taking \(\tau\log\) on both sides yields
\begin{align}
\tau\log Z_h(z,a_{\mathrm{prev}})
&=
\tau\log
\sum_{b\in\mathcal C(z,a_{\mathrm{prev}})}
\exp\!\left(
\frac{
q(z,b)+V_{h-1}^{\mathrm{lse}}(f(z,b),b)
}{\tau}
\right)\\
&=
V_h^{\mathrm{lse}}(z,a_{\mathrm{prev}}),
\end{align}
where the last equality is exactly the LSE backup recursion in Eq.~\eqref{eq:lse_backup} with \(\gamma=1\). 
This completes the induction.
\end{proof}

\paragraph{Remark on discounted lookahead.}
The theorem above states the exact path-partition interpretation for \(\gamma=1\). 
When \(\gamma<1\), Eq.~\eqref{eq:lse_backup} is still a valid soft Bellman recursion, but it is no longer equal to a simple log-partition over additively discounted path scores with a fixed temperature. 
The reason is that the recursive term appears as
\[
\gamma V_{h-1}^{\mathrm{lse}}(f(z,b),b),
\]
and exponentiating it gives
\[
\exp\!\left(
\frac{\gamma V_{h-1}^{\mathrm{lse}}(f(z,b),b)}{\tau}
\right)
=
\left[
\exp\!\left(
\frac{V_{h-1}^{\mathrm{lse}}(f(z,b),b)}{\tau}
\right)
\right]^\gamma.
\]
Thus the suffix partition function is raised to the power \(\gamma\), rather than simply contributing an additive discounted path score inside one global partition function. 
Equivalently, \(\gamma<1\) can be viewed as a heuristic local discount or risk/temperature modulation of future path mass. 
In our use of LPDP, the partition-function statement is therefore used to interpret the undiscounted LSE backup, while the discounted case is treated as the corresponding soft dynamic-programming variant with attenuated future correction.

\subsection{Theorem~\ref{thm:max_limit}: LPDP-Max as the low-temperature limit}
\label{app:proof_max_limit}

\begin{theorem}[LPDP-Max is the low-temperature limit]
\label{thm:max_limit}
Assume all local candidate sets are finite and nonempty along the depth-\(h\) local graph. 
For any fixed horizon \(h\), local state \(z\), and previous edit \(a_{\mathrm{prev}}\),
\[
\lim_{\tau\to 0^+}
V_h^{\mathrm{lse}}(z,a_{\mathrm{prev}})
=
V_h^{\max}(z,a_{\mathrm{prev}}).
\]
Moreover, when \(\gamma=1\),
\[
0\le 
V_h^{\mathrm{lse}}(z,a_{\mathrm{prev}})
-
V_h^{\max}(z,a_{\mathrm{prev}})
\le
\tau\log|\mathcal T_h(z,a_{\mathrm{prev}})|.
\]
\end{theorem}

\begin{proof}
We use the elementary finite-dimensional bound: for any finite vector \(u=(u_1,\ldots,u_m)\),
\begin{equation}
\label{eq:lse_bound}
    \max_i u_i
    \le
    \tau\log\sum_{i=1}^m \exp(u_i/\tau)
    \le
    \max_i u_i+\tau\log m.
\end{equation}
In particular, for fixed \(u\), the middle term converges to \(\max_i u_i\) as \(\tau\to0^+\).

We first prove the low-temperature limit by induction on \(h\). 
For \(h=0\), both \(V_0^{\mathrm{lse}}\) and \(V_0^{\max}\) are zero, so the claim is immediate. 
Assume the claim holds for horizon \(h-1\). 
For each candidate \(b\in\mathcal C(z,a_{\mathrm{prev}})\), define
\[
u_b(\tau)
=
q(z,b)+\gamma V_{h-1}^{\mathrm{lse}}(f(z,b),b).
\]
By the induction hypothesis,
\[
u_b(\tau)
\to
u_b^0
:=
q(z,b)+\gamma V_{h-1}^{\max}(f(z,b),b)
\qquad
\text{as }\tau\to0^+.
\]
Because the candidate set is finite,
\[
\max_{b} u_b(\tau)\to \max_b u_b^0.
\]
Applying Eq.~\eqref{eq:lse_bound} to the finite vector \(\{u_b(\tau)\}_{b\in\mathcal C(z,a_{\mathrm{prev}})}\), we obtain
\[
\max_b u_b(\tau)
\le
V_h^{\mathrm{lse}}(z,a_{\mathrm{prev}})
\le
\max_b u_b(\tau)
+
\tau\log|\mathcal C(z,a_{\mathrm{prev}})|.
\]
Taking \(\tau\to0^+\) gives
\[
\lim_{\tau\to0^+}
V_h^{\mathrm{lse}}(z,a_{\mathrm{prev}})
=
\max_{b\in\mathcal C(z,a_{\mathrm{prev}})}
\left[
q(z,b)+\gamma V_{h-1}^{\max}(f(z,b),b)
\right]
=
V_h^{\max}(z,a_{\mathrm{prev}}).
\]
This proves the first claim.

For the finite-gap bound, assume \(\gamma=1\). 
By Theorem~\ref{thm:path_partition}, the LSE value can be written as the log-partition function over local path scores:
\[
V_h^{\mathrm{lse}}(z,a_{\mathrm{prev}})
=
\tau\log
\sum_{\xi\in\mathcal T_h(z,a_{\mathrm{prev}})}
\exp(G_h(\xi;z,a_{\mathrm{prev}})/\tau).
\]
Under the same unrestricted-temperature path view, the Max value is the maximum path score:
\[
V_h^{\max}(z,a_{\mathrm{prev}})
=
\max_{\xi\in\mathcal T_h(z,a_{\mathrm{prev}})}
G_h(\xi;z,a_{\mathrm{prev}}).
\]
Applying Eq.~\eqref{eq:lse_bound} to the finite vector of path scores
\(\{G_h(\xi;z,a_{\mathrm{prev}}):\xi\in\mathcal T_h(z,a_{\mathrm{prev}})\}\) yields
\[
0
\le
V_h^{\mathrm{lse}}(z,a_{\mathrm{prev}})
-
V_h^{\max}(z,a_{\mathrm{prev}})
\le
\tau\log|\mathcal T_h(z,a_{\mathrm{prev}})|.
\]
This proves the gap bound.
\end{proof}

\subsection{Proof of Proposition~\ref{prop:after_topk}}
\label{app:proof_after_topk}

The ST-after set \(\mathcal C_{\mathrm{after}}\) is a subset of \(\widetilde{\mathcal C}\) by definition. The set \(\widetilde{\mathcal C}\) is exactly the top-\(K_{\mathrm{loc}}\) set under \(p_0(b\mid z,t)\) among all actions in \(\mathcal N_r(z,a_{\mathrm{prev}})\). Therefore every element of \(\mathcal C_{\mathrm{after}}\) has mixed-type local rank at most \(K_{\mathrm{loc}}\). In contrast, ST-first applies top-\(K\) only after restricting to the same-type subset, so an action may be high-ranked within the same-type subset while having mixed-type rank greater than \(K_{\mathrm{loc}}\). ST-first is therefore the more aggressive typed rule: it can bring in candidates whose mixed-type rank lies outside the base shortlist.

\subsection{Same-type diagnostics at the 256/16 operating point}
\label{app:sametype_diagnostics}

We diagnose same-type restrictions as typed approximations to mixed local DP. The goal is to measure whether a typed local graph can reduce the local search space while preserving the root decisions and base-shortlist faithfulness of the mixed graph.

\paragraph{Branch reduction.}
Let \(\mathcal C_{\mathrm{mixed}}=\mathcal C_{\mathrm{none}}\) and let \(\mathcal C_\rho\) be the candidate set induced by a typed rule \(\rho\in\{\mathrm{after},\mathrm{first}\}\). We report the relative candidate size
\begin{equation}
    \mathrm{CandRatio}_{\rho}
    =
    \frac{|\mathcal C_\rho(z,a_{\mathrm{prev}})|}{|\mathcal C_{\mathrm{mixed}}(z,a_{\mathrm{prev}})|},
\end{equation}
and the analogous enumerated local path-count ratio. Values below one indicate that the typed rule shrinks the local search space.

\paragraph{Root-decision agreement with mixed local DP.}
For a set of retained root edits, define mixed and typed LPDP scores
\begin{equation}
    S_{\mathrm{mixed}}(x_t,a)
    =
    q_t(x_t,a)+\lambda V_{H-1}^{\mathrm{mixed}}(f(x_t,a),a),
\end{equation}
\begin{equation}
    S_{\rho}(x_t,a)
    =
    q_t(x_t,a)+\lambda V_{H-1}^{\rho}(f(x_t,a),a).
\end{equation}
We report the top-1 agreement
\begin{equation}
    \mathbf 1\left[
    \arg\max_a S_{\rho}(x_t,a)
    =
    \arg\max_a S_{\mathrm{mixed}}(x_t,a)
    \right]
\end{equation}
for both Max and LSE backups. This measures whether typed local re-solving preserves the root decision of mixed local DP.

\paragraph{Mixed-rank faithfulness.}
For each candidate retained by a typed rule, we compute its rank under the mixed local base-prior ordering in \(\mathcal N_r(z,a_{\mathrm{prev}})\). The mixed-rank tail is
\begin{equation}
    \Pr\left[\operatorname{rank}_{\mathrm{mixed}}(b)>K_{\mathrm{loc}}\right].
\end{equation}
For ST-after, this value is zero by construction. For ST-first, it can be positive because type filtering happens before ranking.

\paragraph{LSE mass efficiency.}
For ST-after, the typed graph is a subset of the mixed local top-\(K\) graph. We therefore define the relative soft path-mass efficiency
\begin{equation}
    \mathrm{MassEff}_{\mathrm{after}}
    =
    \frac{Z_{\mathrm{after}}/Z_{\mathrm{mixed}}}{|\mathcal T_{\mathrm{after}}|/|\mathcal T_{\mathrm{mixed}}|}.
\end{equation}
Values above one indicate that the retained typed paths carry more LSE mass than expected from their path-count fraction. For ST-first, this ratio is a relative soft-value diagnostic rather than a subset coverage guarantee, because ST-first can introduce candidates outside the mixed top-\(K\) shortlist.

\begin{figure}[t]
    \centering
    \includegraphics[width=0.98\linewidth]{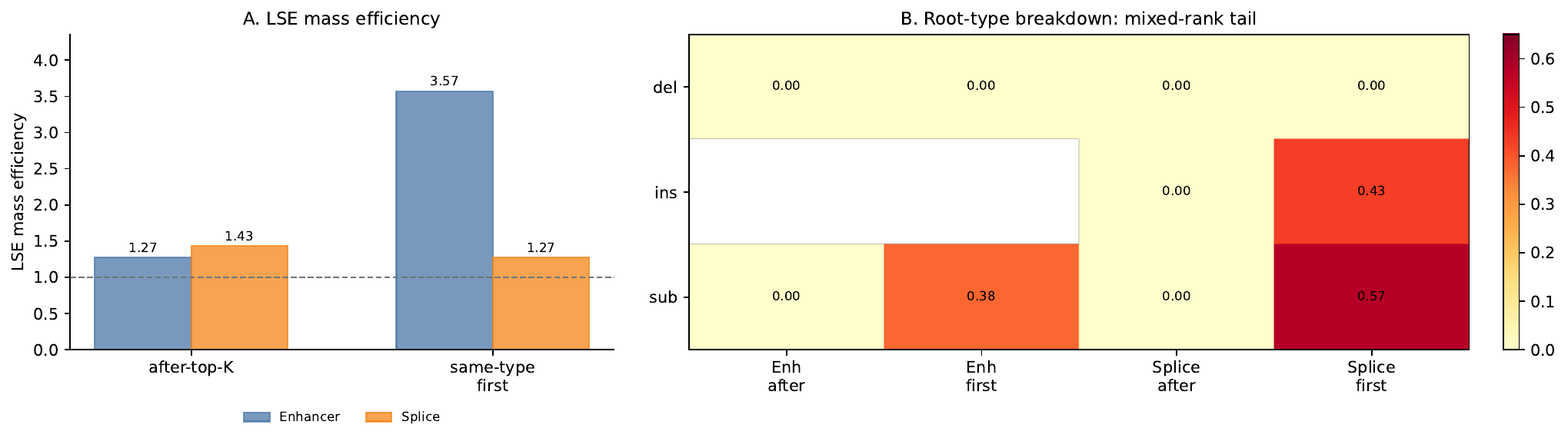}
    \caption{\textbf{Additional same-type diagnostics.}
    LSE mass efficiency measures retained soft path mass normalized by retained local path count. The mixed-rank tail heatmap shows where typed rules select candidates outside the mixed local top-\(K\) shortlist. ST-after has zero tail by construction, while ST-first can introduce lower mixed-rank candidates in a root-type-dependent way.}
    \label{fig:sametype_diagnostics}
\end{figure}

\subsection{Same-type pruning under Max and LSE}
\label{app:proof_pruning}

\begin{proposition}[Exact recovery under harmless same-type pruning]
\label{prop:sametype_recovery}
Fix a local state \(z\), previous edit \(a_{\mathrm{prev}}\), and horizon \(h\). Let \(\mathcal C_{\mathrm{mixed}}\) be a mixed-type local candidate rule and \(\mathcal C_{\mathrm{same}}\subseteq\mathcal C_{\mathrm{mixed}}\) be a same-type restricted rule at every local state. For the Max backup, suppose there exists an optimal length-\(h\) local path under \(\mathcal C_{\mathrm{mixed}}\) whose every edit remains feasible under \(\mathcal C_{\mathrm{same}}\). Then \(V_h^{\max,\mathrm{same}}(z,a_{\mathrm{prev}})=V_h^{\max,\mathrm{mixed}}(z,a_{\mathrm{prev}})\).
\end{proposition}

\begin{proposition}[Soft pruning gap]
\label{prop:soft_pruning_gap}
Assume \(\gamma=1\) and finite local path sets. Let \(\mathcal T_h^{\mathrm{mixed}}\) and \(\mathcal T_h^{\mathrm{same}}\) be the path sets induced by a mixed-type rule and a same-type restricted rule, with \(\mathcal T_h^{\mathrm{same}}\subseteq\mathcal T_h^{\mathrm{mixed}}\). Define \(Z_{\mathrm{mixed}}=\sum_{\xi\in\mathcal T_h^{\mathrm{mixed}}}\exp(G_h(\xi)/\tau)\) and \(Z_{\mathrm{same}}\) analogously. Then
\begin{equation}
\label{eq:soft_pruning_gap}
0\le V_h^{\mathrm{lse},\mathrm{mixed}}(z,a_{\mathrm{prev}})-V_h^{\mathrm{lse},\mathrm{same}}(z,a_{\mathrm{prev}})=\tau\log\frac{Z_{\mathrm{mixed}}}{Z_{\mathrm{same}}}.
\end{equation}
\end{proposition}

This subset statement applies directly to ST-after, whose candidate graph is contained in the mixed local top-\(K\) graph. For ST-first, the analogous relative LSE value is a comparison between two different local graphs, because ST-first can select candidates outside the mixed local top-\(K\) shortlist.

\begin{proof}[Proof of Proposition~\ref{prop:sametype_recovery}]
Because
\begin{equation}
    \mathcal C_{\mathrm{same}}(z',a')
    \subseteq
    \mathcal C_{\mathrm{mixed}}(z',a')
\end{equation}
for every local state and previous edit in the recursion, the feasible path set under same-type pruning is a subset of the feasible path set under the mixed-type local graph.  Therefore,
\begin{equation}
    V_h^{\max,\mathrm{same}}(z,a_{\mathrm{prev}})
    \le
    V_h^{\max,\mathrm{mixed}}(z,a_{\mathrm{prev}}).
\end{equation}
By assumption, there exists a mixed-type optimal path
$\xi^\star=(b_0^\star,\ldots,b_{h-1}^\star)$ whose every edit remains feasible under the same-type restricted candidate sets.  Hence $\xi^\star$ is feasible for the same-type problem.  Let $G_h(\xi^\star)$ be its local path score.  Then
\begin{equation}
    V_h^{\max,\mathrm{same}}(z,a_{\mathrm{prev}})
    \ge
    G_h(\xi^\star)
    =
    V_h^{\max,\mathrm{mixed}}(z,a_{\mathrm{prev}}).
\end{equation}
Combining the two inequalities gives equality.
\end{proof}

\begin{proof}[Proof of Proposition~\ref{prop:soft_pruning_gap}]
By Theorem~\ref{thm:path_partition}, the LSE values can be written as local path partition functions:
\begin{equation}
    V_h^{\mathrm{lse},\mathrm{mixed}}(z,a_{\mathrm{prev}})
    =
    \tau\log Z_{\mathrm{mixed}},
    \qquad
    V_h^{\mathrm{lse},\mathrm{same}}(z,a_{\mathrm{prev}})
    =
    \tau\log Z_{\mathrm{same}}.
\end{equation}
Since $\mathcal T_h^{\mathrm{same}}\subseteq\mathcal T_h^{\mathrm{mixed}}$ and every term $\exp(G_h(\xi)/\tau)$ is nonnegative, we have
\begin{equation}
    Z_{\mathrm{mixed}}\ge Z_{\mathrm{same}}>0.
\end{equation}
Therefore,
\begin{align}
    V_h^{\mathrm{lse},\mathrm{mixed}}(z,a_{\mathrm{prev}})
    -
    V_h^{\mathrm{lse},\mathrm{same}}(z,a_{\mathrm{prev}})
    &=
    \tau\log Z_{\mathrm{mixed}}
    -
    \tau\log Z_{\mathrm{same}} \\
    &=
    \tau\log\frac{Z_{\mathrm{mixed}}}{Z_{\mathrm{same}}}
    \ge 0.
\end{align}
This proves the stated identity and nonnegativity.
\end{proof}

\section{Task-specific instantiations}
\label{app:instantiations}

Across tasks, LPDP uses the same root re-solving recursion. What changes is the conditioning context, the frozen oracle used to score candidate sequences, and the rollout window where reward guidance is applied. This section clarifies these task-level choices so that the method section can remain task-agnostic.

\paragraph{Enhancer LPDP-Edit.}
For enhancer optimization, the state is a variable-length DNA sequence and the action is an explicit edit
\begin{equation}
    a=(s,e,v),
\end{equation}
where \(s\) is the site, \(e\in\{\mathrm{sub},\mathrm{ins},\mathrm{del}\}\) is the edit type, and \(v\) is the inserted or substituted nucleotide when applicable. The frozen gReLU-based oracle scores predicted HepG2 regulatory activity~\citep{lal2025grelu,wang2025drakes}. We apply guidance in the first 16 steps of the 256-step edit-flow schedule because early edits have the largest opportunity to shape the global regulatory sequence pattern before later steps perform local refinements.

\paragraph{Splice LPDP-Inpaint.}
For exon--intron--exon inpainting, the generator is conditioned on left and right exon contexts and generates the intervening intron through the same substitution, insertion, and deletion actions. The frozen SpliceAI oracle scores the intended donor and acceptor junctions~\citep{jaganathan2019spliceai}. We apply guidance in the last 16 steps of the 256-step schedule because the splice reward becomes most actionable once a coarse intron context and candidate donor/acceptor boundary neighborhoods have formed. At each guided step, the algorithm is unchanged: LPDP scores root edits, retains a root band, and solves the bounded typed local graph around each retained root.

\section{Data and metric definitions}
\label{app:data_metrics}

\subsection{Enhancer data construction}
\label{app:enhancer_data}
The enhancer benchmark is constructed from GRCh38-derived candidate regulatory intervals. The sequence pool contains variable-length DNA sequences of 200--400 bp, including dELS-derived candidates and background regions. dELS and background annotations are derived from ENCODE candidate cis-regulatory element resources~\citep{encode2020expanded}. Background intervals exclude annotated cCREs, promoters, and blacklist regions so that they provide a broad non-enhancer reference set.

For each candidate sequence, five 200-bp windows are extracted. If the sequence is longer than 200 bp, windows are placed at uniformly spaced positions and clamped to remain inside the sequence; if the sequence is exactly 200 bp, the same window is repeated. Each window is evaluated by the frozen gReLU activity oracle~\citep{lal2025grelu}. The final sequence-level label is the mean of the top two window predictions. The reported enhancer benchmark optimizes the HepG2 activity output.

\subsection{Splice inpainting data construction}
\label{app:splice_data}

The splice inpainting dataset is constructed from exon--intron--exon triplets extracted from the GRCh38 primary assembly using GENCODE gene annotations~\citep{frankish2021gencode}. For each annotated transcript, consecutive exon pairs and their intervening intron are extracted as \((\texttt{Exon}_{5'},\texttt{Intron},\texttt{Exon}_{3'})\) triplets and oriented to the coding strand. Triplets are filtered to retain canonical GT donor and AG acceptor junctions and introns in the target length range used by the edit-flow inpainting model. Each example provides a left exon context, an intron to be generated, and a right exon context; generated introns are inserted back between the two exons for oracle evaluation.

For reward evaluation, each exon--intron--exon sequence is placed into the input window required by the frozen SpliceAI backend~\citep{jaganathan2019spliceai}. We extract the donor probability at the intended donor junction and the acceptor probability at the intended acceptor junction. The primary scalar reward is the geometric mean of the two intended junction scores,
\begin{equation}
  R_{\mathrm{splice}}(x)
  =
  \sqrt{D(x)A(x)},
  \label{eq:splice_reward}
\end{equation}
where \(D(x)\) and \(A(x)\) denote donor and acceptor probabilities returned by the frozen oracle.

\subsection{Enhancer metrics}
\label{app:metric_definitions}

\paragraph{Predicted activity.}
Predicted activity is the held-out evaluation score of a generated sequence under the frozen HepG2 activity oracle. If \(R_{\mathrm{eval}}\) is the evaluation oracle and \(\{x_i\}_{i=1}^N\) are generated sequences, we report \(N^{-1}\sum_i R_{\mathrm{eval}}(x_i)\).

\paragraph{Calls/sample.}
Calls/sample denotes the average number of uncached reward-oracle evaluations required to produce one final sample, i.e., oracle cache misses rather than total oracle requests. Because repeated queries can be served from cache, this quantity is a cost descriptor rather than an optimization objective and need not increase monotonically with deeper lookahead or larger search trees. In the main benchmark, baseline widths, particle counts, and search depths are chosen to keep methods at a broadly comparable oracle-computation scale, although exact request counts and cache-hit behavior may still vary across methods.

\paragraph{Base trajectory log-likelihood.}
Base trajectory log-likelihood measures how likely the selected edit trajectory is under the frozen base edit-flow model:
\begin{equation}
    \mathrm{BaseTrajLL}
    =
    \frac{1}{N}\sum_{i=1}^N
    \frac{1}{T_i}
    \sum_{t=0}^{T_i-1}
    \log p_0(a_{i,t}\mid x_{i,t},t).
\end{equation}
Higher values indicate trajectories that remain closer to the base generator's edit distribution. This metric complements reward: a method can obtain high activity while moving far away from the base flow, or obtain a more moderate reward gain with higher trajectory plausibility.

\paragraph{3-mer JSD.}
Let \(P_3\) and \(Q_3\) be normalized 3-mer count distributions for generated and reference sequences, and let \(M_3=(P_3+Q_3)/2\). We compute
\begin{equation}
    \mathrm{JSD}_3(P_3,Q_3)
    =
    \frac{1}{2}\mathrm{KL}(P_3\|M_3)+\frac{1}{2}\mathrm{KL}(Q_3\|M_3).
\end{equation}
Lower values indicate smaller global 3-mer composition drift. This is the primary quantitative distributional metric for enhancer experiments.

\paragraph{Pooled reference support.}
For the UMAP visualization in Figure~\ref{fig:enhancer_main}, sequences are embedded using normalized 3-mer frequency vectors followed by a fixed two-dimensional UMAP projection~\citep{mcinnes2018umap}. We fit one density model to the union of top-dELS and background reference embeddings and report the average log density of generated samples under this pooled reference model. This score is a qualitative support diagnostic: it asks whether generated samples occupy high-support regions of the combined reference manifold, not whether they are specifically top-dELS-like. The primary quantitative distributional comparison remains 3-mer JSD in the original 3-mer space.

\paragraph{JASPAR correlation.}
JASPAR correlation compares motif-scan statistics between generated and reference sequences across a panel of transcription-factor motifs~\citep{jaspar2024}. Higher values indicate stronger motif-level similarity to the reference distribution.

\subsection{Splice metrics}
\label{app:splice_metrics}

\paragraph{Splice geomean.}
Splice geomean is the primary scalar splice-quality metric:
\begin{equation}
    \mathrm{SpliceGeomean}(x)=\sqrt{D(x)A(x)}.
\end{equation}
It rewards sequences for which both intended junctions receive high SpliceAI scores. A method must improve both donor and acceptor probabilities to obtain a high geomean.

\paragraph{Splice minimum.}
Splice minimum records the weaker of the two intended splice-site scores,
\begin{equation}
    \mathrm{SpliceMin}(x)=\min\{D(x),A(x)\}.
\end{equation}
It penalizes one-sided solutions where only the donor or only the acceptor becomes strong.

\paragraph{Donor and acceptor scores.}
When detailed tables are shown, we also report \(N^{-1}\sum_i D(x_i)\) and \(N^{-1}\sum_i A(x_i)\) to show whether a method improves both splice junctions or only one side.

\paragraph{Donor GT rate.}
Donor GT rate is the fraction of generated introns whose intended donor boundary has the canonical \texttt{GT} dinucleotide. This is a sequence-level sanity metric complementary to the SpliceAI donor and acceptor probabilities.

\paragraph{Base trajectory log-likelihood.}
Splice base trajectory log-likelihood is computed under the frozen conditional edit-flow inpainting model:
\begin{equation}
    \mathrm{BaseTrajLL}_{\mathrm{splice}}
    =
    \frac{1}{N}\sum_{i=1}^N\frac{1}{T_i}
    \sum_{t=0}^{T_i-1}\log p_0(a_{i,t}\mid x_{i,t},c_i,t),
\end{equation}
where \(c_i\) is the exon context. It should be compared only among splice methods because the conditioning context differs from the enhancer setting.

\paragraph{Calls/sample.}
Calls/sample denotes the average number of uncached frozen SpliceAI reward evaluations required to produce one generated intron, i.e., cache misses rather than total oracle requests. As in enhancer optimization, we report it as a cost descriptor rather than as an objective, and use it to check that methods are compared at a broadly similar oracle-computation scale.

\section{Appendix experiment suites}
\label{app:appendix_suites}

This section reports targeted operating-point ablations for LPDP. 
We keep the LPDP operator fixed---exact root scoring, root-band selection, bounded local graph construction, the Mixed/ST-after/ST-first candidate rules, and the Max/LSE backups---and vary only inference-time choices. 
Rather than exhaustively tuning LPDP, these ablations test whether the main operating point is an isolated choice and examine three claims used in the main text: enhancer rewards are most useful in a front-loaded guidance window, splice rewards are most useful in a back-loaded guidance window, and \(H=2\) is the smallest nontrivial root-local re-solving horizon beyond one-step root scoring.
Across tasks, the defaults remain competitive under the chosen oracle budget, while the trends reveal reward--compute trade-offs.
Calls/sample is reported as a cost descriptor throughout.

\subsection{Guidance-window position}
\label{app:window_position_ablations}
\label{app:splice_ablation_suites}

The first ablation varies \emph{where} the 16 guided rollout steps are placed within the 256-step edit-flow schedule. 
For enhancer optimization, the main text uses a front-loaded window because early edits can shape global regulatory sequence structure. 
For splice inpainting, the main text uses a back-loaded window because the splice oracle is most informative after an intron context and donor/acceptor boundary candidates have formed.

\begin{table}[h]
\centering
\small
\caption{Guidance-window position ablation. The LPDP operator and hyperparameters are fixed; only the location of the 16-step guided window changes. For enhancer, the main reward is predicted HepG2 activity. For splice, the main reward is Splice-Geomean.}
\label{tab:app_window_position_ablation}
\resizebox{0.88\linewidth}{!}{
\begin{tabular}{llccc}
\toprule
Task & LPDP variant & Guided window & Main reward $\uparrow$ & Calls/sample \\
\midrule
Enhancer & LPDP-ST-after-LSE & First-16 & 7.964 & 39,581 \\
Enhancer & LPDP-ST-after-LSE & Middle-16 & 7.577 & 39,329 \\
Enhancer & LPDP-ST-after-LSE & Last-16 & 7.473 & 34,370 \\
\midrule
Splice & LPDP-ST-first-Max & Early-16 & 0.614 & 36,648 \\
Splice & LPDP-ST-first-Max & Middle-16 & 0.756 & 31,243 \\
Splice & LPDP-ST-first-Max & Late-16 & 0.834 & 24,350 \\
\bottomrule
\end{tabular}}
\end{table}

\paragraph{Interpretation.}
The default window positions match the reward geometry of each task. Enhancer guidance is most effective in the first 16 steps, consistent with early edits shaping global regulatory structure. Splice guidance is strongest in the late window, when a coarse intron and candidate donor/acceptor neighborhoods are already present. This supports using task-dependent guidance windows while keeping the LPDP operator itself fixed.

\subsection{Guidance-window length}
\label{app:window_length_ablations}

The second ablation varies \emph{how long} reward guidance is applied while keeping the task-default position fixed: front-loaded for enhancer and back-loaded for splice. 
This tests whether the main 16-step guidance window is a stable operating point rather than a single arbitrary schedule choice.

\begin{table}[h]
\centering
\small
\caption{Guidance-window length ablation. LPDP uses the task-default window position: front-loaded guidance for enhancer and back-loaded guidance for splice.}
\label{tab:app_enhancer_schedule_ablation}
\resizebox{0.82\linewidth}{!}{
\begin{tabular}{llccc}
\toprule
Task & LPDP variant & Guided steps & Main reward $\uparrow$ & Calls/sample \\
\midrule
Enhancer & LPDP-ST-after-LSE & First-8 & 6.661 & 19,690 \\
Enhancer & LPDP-ST-after-LSE & First-16 & 7.964 & 39,581 \\
Enhancer & LPDP-ST-after-LSE & First-32 & 9.044 & 78,979 \\
\midrule
Splice & LPDP-ST-first-Max & Late-8 & 0.637 & 10,615 \\
Splice & LPDP-ST-first-Max & Late-16 & 0.834 & 24,350 \\
Splice & LPDP-ST-first-Max & Late-32 & 0.835 & 40,545 \\
\bottomrule
\end{tabular}}
\end{table}

\paragraph{Interpretation.}
Increasing the guidance window improves enhancer reward, but also increases oracle cost almost proportionally; the 16-step window is therefore used as a cost-controlled main operating point rather than the maximum-reward setting. For splice inpainting, the reward saturates from 16 to 32 late guided steps, so the 16-step window achieves nearly the same splice score with substantially fewer uncached oracle calls. These trends suggest that the default guided windows are not arbitrary single points, but practical reward--compute trade-offs.

\subsection{Future-correction strength}
\label{app:strength_ablations}

The LPDP score adds a local future correction to the one-step root score, scaled by \(\lambda\). 
We vary \(\lambda\) while keeping the root band, local graph, candidate rule, backup, and guided window fixed. 
This tests whether the reported operating point is robust to the relative weight assigned to the local continuation value.

\begin{table}[h]
\centering
\small
\caption{LPDP future-correction strength ablation. The LPDP operator is fixed and only the strength \(\lambda\) of the local future correction is varied.}
\label{tab:app_enhancer_strength_ablation}
\resizebox{0.75\linewidth}{!}{
\begin{tabular}{llccc}
\toprule
Task & LPDP variant & \(\lambda\) & Main reward $\uparrow$ & Calls/sample \\
\midrule
Enhancer & LPDP-ST-after-LSE & 0.25 & 7.747 & 38,851 \\
Enhancer & LPDP-ST-after-LSE & 0.5 & 7.964 & 39,581 \\
Enhancer & LPDP-ST-after-LSE & 1.0 & 7.781 & 39,455 \\
\midrule
Splice & LPDP-ST-first-Max & 0.25 & 0.825 & 26,753 \\
Splice & LPDP-ST-first-Max & 0.5 & 0.834 & 24,350 \\
Splice & LPDP-ST-first-Max & 1.0 & 0.653 & 23,176 \\
\bottomrule
\end{tabular}}
\end{table}

\paragraph{Interpretation.}
Moderate future-correction strength is the most reliable setting across tasks. Setting \(\lambda=0.5\) improves over a weaker correction in both tasks, while \(\lambda=1.0\) can over-trust the truncated local graph, especially in splice inpainting. This supports interpreting \(\lambda\) as a damping coefficient for an approximate local future value rather than as an independent reward scale.

\subsection{Local lookahead horizon}
\label{app:horizon_ablations}

The final ablation varies the LPDP horizon \(H\). 
At \(H=1\), the local continuation term is zero and LPDP reduces to exact one-step root scoring over the retained root band. 
At \(H=2\), each retained root receives one local continuation step, which is the main setting and the smallest nontrivial re-solving horizon. 
Larger \(H\) values test whether deeper local lookahead improves reward enough to justify the additional oracle cost.

\begin{table}[h]
\centering
\small
\caption{Local lookahead horizon ablation. \(H=1\) reduces LPDP to one-step root scoring, while \(H=2\) is the main minimal nontrivial root-local re-solving horizon.}
\label{tab:app_splice_horizon_ablation}
\resizebox{0.82\linewidth}{!}{
\begin{tabular}{llccc}
\toprule
Task & LPDP variant & Horizon \(H\) & Main reward $\uparrow$ & Calls/sample \\
\midrule
Enhancer & LPDP-ST-after-LSE & 1 & 7.762 & 38,909 \\
Enhancer & LPDP-ST-after-LSE & 2 & 7.964 & 39,581 \\
Enhancer & LPDP-ST-after-LSE & 3 & 7.839 & 39,421 \\
\midrule
Splice & LPDP-ST-first-Max & 1 & 0.761 & 21,293 \\
Splice & LPDP-ST-first-Max & 2 & 0.834 & 24,350 \\
Splice & LPDP-ST-first-Max & 3 & 0.859 & 24,154 \\
\bottomrule
\end{tabular}}
\end{table}

\paragraph{Interpretation.}
The horizon ablation separates one-step root scoring from local re-solving. Moving from \(H=1\) to \(H=2\) improves the main reward in both tasks, showing that the local future correction contributes beyond exact one-step scoring. Deeper lookahead is not uniformly better: \(H=3\) slightly improves splice reward but does not improve enhancer reward in this diagnostic, and it introduces a larger local search tree whose realized uncached call count depends on cache reuse. We therefore use \(H=2\) as the main minimal nontrivial horizon, while viewing deeper budget-aware re-solving as a natural extension.

\section{Design assumptions and negative ablations}
\label{app:negative_results}
\label{app:design_assumptions}

LPDP deliberately uses a small amount of structured lookahead rather than a large global trajectory search. 
This section records the main design assumptions and the negative ablations that shaped the final method.

\paragraph{Root-band approximation.}
LPDP allocates local lookahead only to roots whose exact one-step tilted score lies near the best root. 
This does not assert that the globally best finite-horizon trajectory must start inside the band. 
Instead, it is an oracle-allocation rule: exact root scoring identifies competitive edits, and the expensive local DP is spent only around those roots.

\paragraph{Root-band truncation.}
The root band is a budgeted approximation: LPDP spends local lookahead only on roots that are competitive under exact one-step tilted scoring. 
This may miss a root with low immediate score but unusually favorable long-range continuation value. 
We therefore interpret the horizon and guidance-window ablations as operating-point diagnostics rather than as a global optimality guarantee.

\paragraph{Locality assumption.}
The local perturbation graph assumes that the most informative short continuations of a retained root are found near the most recent edit site. 
This is an action-geometry assumption of edit flows, not a guarantee about long-horizon global optimality. 
The operating-point ablations above keep this local graph construction fixed and vary when, how long, and how strongly the local re-solving correction is applied.

\paragraph{Minimal nontrivial horizon.}
The main LPDP setting uses local depth \(H=2\). 
At \(H=1\), LPDP reduces to exact one-step root selection. 
At \(H=2\), each retained root receives one local continuation step, which is the smallest setting that can re-rank roots using future information. 
Beam baselines use depth 2, giving them a broader sequence-level frontier over the same two-edit horizon, while LPDP restricts the second step to a root-conditioned local graph.

\paragraph{Aggressive future twists.}
We tested increasingly aggressive online twist and lookahead variants in enhancer sanity checks. 
Conservative settings often tied exact one-step planning, whereas aggressive twists degraded reward or cost-efficiency. 
These results motivated the final emphasis on root-band selection plus bounded local graph DP rather than more elaborate future-twist heuristics.

\section{Reproducibility details}
\label{app:reproducibility}

All experiments use frozen base edit-flow generators and frozen reward oracles. Main tables use \(N=500\) generated samples per method, while appendix ablations use smaller sample counts when only diagnostic trends are needed.

\paragraph{Compute resources.}
All base edit-flow generators and inference experiments were run on a single NVIDIA RTX 3090 GPU.
Training the enhancer edit-flow generator for 5{,}000 steps took approximately 48 hours, and training the splice edit-flow generator for 5{,}000 steps took approximately 72 hours.
After training, all reward-guided experiments use frozen generators and frozen reward oracles.
For inference, the main enhancer setting takes approximately 1 minute per generated sample, while the main splice setting takes approximately 3 minutes per generated sample, with variation across baselines depending on oracle-call count and caching.
We report calls/sample as the average number of uncached reward-oracle evaluations used to generate one final sample.

\paragraph{Comparable search budgets.}
Baseline parameters are selected to match the LPDP horizon and oracle-call scale as closely as possible rather than to tune an oracle-call Pareto frontier. Beam uses depth 2 to match LPDP's two-edit lookahead horizon \(H=2\); CEM, TDS, and SMC use population or particle counts chosen to yield comparable calls/sample in the main tables. Calls/sample is therefore reported as a cost descriptor for each method.

\paragraph{Enhancer defaults.}
The main enhancer LPDP setting uses reward weight \(\beta=20.0\), root-band width \(\delta=2.0\), root cap \(K_{\mathrm{root}}=16\), local depth \(H=2\), local radius \(r=1\), local cap \(K_{\mathrm{loc}}=8\), LSE temperature \(\tau=1.0\), discount \(\gamma=1.0\), and future-correction strength \(\lambda=0.5\). The default guided schedule is \(256/16\): 256 base edit-flow steps with reward guidance applied during the first 16 rollout steps.

\paragraph{Splice defaults.}
The splice experiments use the same LPDP recursion, candidate rules, and backup definitions, but with exon-conditioned edit-flow states and a frozen SpliceAI reward. The default guided schedule is also based on a 256-step edit-flow rollout, with reward guidance applied during the last 16 rollout steps. This late window is used because donor/acceptor boundary candidates are more meaningful once a coarse intron context has formed.

\begin{table}[h]
\centering
\small
\caption{Default LPDP and baseline settings for the enhancer benchmark.}
\label{tab:repro_lpdp_defaults}
\resizebox{0.7\linewidth}{!}{
\begin{tabular}{lc}
\toprule
Parameter & Value \\
\midrule
Samples per method & 500 \\
Total edit-flow steps & 256 \\
Reward-guided rollout steps & 16 \\
Reward weight \(\beta\) & 20.0 \\
Root-band width \(\delta\) & 2.0 \\
Root cap \(K_{\mathrm{root}}\) & 16 \\
Local depth \(H\) & 2 \\
Local radius \(r\) & 1 \\
Local cap \(K_{\mathrm{loc}}\) & 8 \\
LSE temperature \(\tau\) & 1.0 \\
Discount \(\gamma\) & 1.0 \\
LPDP strength \(\lambda\) & 0.5 \\
CEM samples / elites / rounds & \(64/8/2\) \\
Beam width / depth & \(8/2\) \\
TDS temperature / support top-\(K\) & \(1.0/8\) \\
SMC particles / depth / proposal top-\(K\) & \(64/2/32\) \\
\bottomrule
\end{tabular}
}
\end{table}



\end{document}